\documentclass{article}
\pdfoutput=1
\usepackage[preprint,nonatbib]{neurips_2021}

\usepackage[utf8]{inputenc} %
\usepackage[T1]{fontenc}    %
\usepackage{url}            %
\usepackage{booktabs}       %
\usepackage{amsfonts}       %
\usepackage{nicefrac}       %
\usepackage{microtype}      %
\usepackage{xcolor}         %

\usepackage{newfloat}
\usepackage{listings}
\usepackage{tikz,tikz-3dplot} 
\usepackage[ruled,vlined]{algorithm2e}
\usepackage{subfig}

\usepackage{bbm}
\usetikzlibrary{intersections,fadings,decorations.pathreplacing}

\usepackage{ekzhang}
\usepackage{verbatim}
\usepackage{bm}

\newcommand{\row}{\operatorname{Row}}
\newcommand{\col}{\operatorname{Col}}
\newcommand{\nul}{\operatorname{Null}}

\newcommand{\bra}[1]{\left[#1\right]}
\newcommand{\pa}[1]{\left(#1\right)}
\newcommand{\set}[1]{\left\{ #1 \right\}}

\newcommand{\ex}[1]{\mathbb{E}\bra{#1}}

\newcommand{\B}[1]{\mathbf{\boldsymbol{#1}}}
\newcommand{\cyc}[1]{\left\langle #1 \right\rangle}

\newcommand{\norm}[1]{\left\lVert#1\right\rVert}

\newcommand{\appropto}{\mathrel{\vcenter{
  \offinterlineskip\halign{\hfil$##$\cr
    \propto\cr\noalign{\kern2pt}\sim\cr\noalign{\kern-2pt}}}}}

\title{Recommending with Recommendations}

\author{%
 Naveen Durvasula\thanks{The first two authors contributed equally} \\
 University of California, Berkeley\\
 Department of EECS\\
 Berkeley, CA 94704 \\
 \texttt{ndurvasula@berkeley.edu} \\
 \And
 Franklyn Wang$^{*}$ \\
 Harvard University \\
 Department of Mathematics\\
 Cambridge, MA 02138 \\
 \texttt{franklyn\_wang@college.harvard.edu} \\
 \And
 Scott Duke Kominers\\
 Harvard Business School \\
 Boston, MA 02138 \\
 \texttt{kominers@fas.harvard.edu} \\
}

\begin{document}

\maketitle

\begin{abstract}

Recommendation systems are a key modern application of machine learning, but they have the downside that they often draw upon sensitive user information in making their predictions. We show how to address this deficiency by basing a service's recommendation engine upon recommendations from other existing services, which contain no sensitive information by nature. Specifically, we introduce a contextual multi-armed bandit recommendation framework where  the agent has access to recommendations for other services. In our setting, the user's (potentially sensitive) information belongs to a high-dimensional latent space, and the ideal recommendations for the source and target tasks (which are non-sensitive) are given by unknown linear transformations of the user information. So long as the tasks rely on similar segments of the user information, we can decompose the target recommendation problem into systematic components that can be derived from the source recommendations, and idiosyncratic components that are user-specific and cannot be derived from the source, but have significantly lower dimensionality. We propose an explore-then-refine approach to learning and utilizing this decomposition; then using ideas from perturbation theory and statistical concentration of measure, we prove our algorithm achieves regret comparable to a strong skyline that has full knowledge of the source and target transformations. We also consider a generalization of our algorithm to a model with many simultaneous targets and no source. Our methods obtain superior empirical results on synthetic benchmarks.

\end{abstract}

\pdfoutput=1

\section{Introduction}

Recommendation engines enable access to a range of products by helping consumers efficiently explore options for everything from restaurants to movies. But these systems typically rely on fine-grained user data---which both necessitates users' sharing such data, and makes it difficult for platforms without existing data sources to enter the space. In this paper, we propose a new approach to recommendation system design that partially mitigates both of these challenges: \textit{basing recommendations on existing recommendations developed by other services}. We call the new service's  recommendation task the \emph{target}, and call the prior recommendations we rely on the \emph{source}. Source recommendations are always available to the user, and are sometimes even public. We show that as long as the source and target recommendation tasks are sufficiently related, users can enable services to make high-quality recommendations for them without giving away any new personal data. Thus our method offers a middle ground between consumers' current options of either giving platforms personal data directly or receiving low-quality recommendations: under our system, consumers can receive high-quality recommendations while only sharing existing recommendations based on data that other services already possess.

For intuition, consider the following example which we use throughout the paper for illustration: There is an existing grocery delivery service, which has become effective at recommending grocery items to its users. A new takeout delivery service is seeking to recommend restaurants to the same set of users. Users' preferences over groceries and restaurant food rely on many of the same underlying user characteristics: for example, a user who loves pasta will buy it at grocery stores \textit{and}  patronize Italian restaurants; meanwhile, a user who is lactose intolerant will prefer not to buy milk products \textit{and} prefer to avoid ice cream parlors. Thus, high-quality grocery recommendations provide valuable data upon which we could base restaurant recommendations. 

From the user perspective, a restaurant recommendation system based on existing grocery recommendations would be able to provide high-quality restaurant suggestions immediately, rather than needing to learn the user's preferences over time. And yet the user would not need to share any data with the restaurant recommendation system other than the grocery recommendations themselves---which by nature contain (weakly) less user private information than they were derived from. Crucially, the user's food preferences might theoretically be driven by a health condition like lactose intolerance, but they could also just be a matter of taste (not liking milk). In this case, the grocery delivery platform's recommendations reflect that preference information in a way that is useful for extrapolating restaurant recommendations, without directly revealing the potentially sensitive underlying source of the user's preferences.

And while these food recommendation instances are a bit niche, there are many settings in which recommendations built on personal data may be available, even if we do not have access to the personal data itself. Search engines, for example, utilize user data to serve advertisements, 
and at least two notable projects---the Ad Observatory project \cite{noauthor_ad_nodate} and the Citizen Browser \cite{noauthor_citizen_nodate}---have automatically logged social media advertisements that were served to users. Such advertisements could be used as contextual input for a variety of recommendation problems.

In our framework, we %
consider an agent that aims to make recommendations for a collection of users over a collection of \textit{target} tasks. Each task takes the form of a standard stochastic linear bandit problem, where the reward parameter for user $u$ is given by a task-specific linear transformation of $u$'s high-dimensional personal data vector. We think of these transformations as converting the personal data into some lower-dimensional features that are relevant for the task at hand. These transformations and the personal data vectors are unknown to the agent. However, the agent is given access to the optimal arms (and its final reward) for each of the users on a \textit{source} task. No other ``white-box'' information about the source task (e.g., the reward history for the agent that solved the task) is known. That is, only the recommendations themselves---which are served to (and thus may be captured by) the user---in addition to the final reward (which may also be captured by asking the user to rate the recommendation) are accessible to the agent.

We describe algorithms that allow the agent to effectively make use of this information to reduce regret, so long as the source and target task transformations make use of segments of the high-dimensional personal data similar to those used in the source task. 
Notably, such a recommendation system functions without the user providing any personal information to the agent directly---and moreover, to the extent that the source task abstracts from sensitive data, all of the target task recommendations do as well.  

Formally, we set up two related contextual bandit problems for the recommendation problem. When we have one target and many users, we can decompose each user's reward parameter into idiosyncratic components, which are unique to each user, and systematic components, which are shared among all users. If the source and target tasks utilize similar segments of the users' personal data, the idiosyncratic component is low-dimensional. As a result, if $0 \le \kappa \le 1$ is the shared information between the subspaces, our regret bounds improve upon the LinUCB baseline \cite{LinUCB} by a factor of $1 - \kappa$.

In our second iteration, we consider the general problem of having many targets and many users, but no sources. We show that if we perform what is essentially an alternating minimization algorithm to solve for a few tasks at first, use the solutions to solve for all the users, and then use those to solve for all the tasks, we obtain strong experimental performance. We call our collective methods \textbf{Rec2}, for \emph{recommending with recommendations}. 

In both analyses, we make a crucial \emph{user growth assumption}, which states that in several initial phases, we have beta groups where only some subset of the tasks and some subset of the users can appear. This assumption enables us to estimate systematic parameters accurately, as there are fewer idiosyncratic parameters. It also reflects how a web platform grows.

\Cref{sec:prelims} introduces our formal model and compares it to prior work. 
In \Cref{sec:model2}, we describe our approach for inferring the relationship between the source and target tasks when we have source recommendations for multiple users. We then develop and analyze an algorithm that uses this approach to reduce regret. In \Cref{sec:model3}, we generalize our algorithm from \Cref{sec:model2} to our second setting where we have multiple target tasks. In \Cref{sec:exp}, we demonstrate our methods in synthetic experiments. Finally, \Cref{sec:concl} concludes.

\pdfoutput=1

\section{Preliminaries}\label{sec:prelims}

\subsection{Framework} \label{sec:prelims-model}
In our setup, we aim to solve $T$ \textit{target} tasks for $U$ \textit{users}. We are also given the optimal arms (and corresponding expected reward) for each user on a \textit{source} task.

Formally, we associate to each user $u \in [U]$ a vector $\B \theta^u \in \R^n$ denoting all of the user's data, some of which is potentially sensitive. We think of this space as being extremely high-dimensional---i.e., $n \gg 1$---and encapsulating all of the salient attributes about the user. We associate to the source task a matrix $A \in \R^{a \times n}$ and similarly associate to each target task $t \in [T]$ a matrix $B_t \in \R^{b \times n}$. The parameters $a, b \le n$ denote the dimensionality of the source target tasks, respectively.\footnote{The assumption that all target tasks have the same dimensionality is made for expositional simplicity and can be relaxed.} We assume that the task matrices $A$ and $B_t$ do not contain any redundant information, and thus have full row rank.  These task matrices transform a user's personal data into the reward parameter for the task. For example, if a takeout delivery service (given by target task $t$) aims to make recommendations for some user $u$, the task matrix $B_t$ transforms $u$'s personal data into the space of possible restaurants in such a way that the vector $B_t \B \theta^u$ essentially denotes $u$'s ideal restaurant choice. 

\begin{figure}[tbp]
    \centering
    \scalebox{0.5}{\pdfoutput=1

\tdplotsetmaincoords{80}{10}
\begin{tikzpicture}[tdplot_main_coords]
\tdplotsetrotatedcoords{0}{1}{0}
\begin{scope}[tdplot_rotated_coords]
\coordinate (O) at (0,0,0);
\filldraw[
        draw=green,%
        fill=green!40,%
    ] (-4,-3,3) -- (-4,3,3) -- (4,3,3) -- (4,-3,3) -- cycle;

\path (0,0,3) node {$\B \theta^u$};
\path (6,3.5,3) node {$\R^n$ -- User Data};
\path (6,3.5,1) node {$\R^r$ -- Low Rank Model};
\path (6,3.5,.6) node {\small e.g. Food Preferences};
\draw[dashed, ->] (0,0,2.8) -- (-1.5, 0, 0);
\draw[dashed, ->] (0,0,2.8) -- (1.5, 0, 0);
\path (-3,1,-1) node {$A^+A \B \theta^u$};
\path (2.5,1,-1) node {$B^+B \B \theta^u$};
\draw[dashed, ->] (-1.5, 0, 0) -- (-3, 0, -3.5);
\draw[dashed, ->] (1.5, 0, 0) -- (3, 0, -3.5);
\draw[thick] (-4,4,1) -- (-4,-4,1) -- (4,-4,1) -- (4,4,1) -- cycle;
\draw[thick, dashed] (-4,4,-1) -- (-4,4,1);
\draw[thick] (4,4,-1) -- (4,4,1);
\draw[thick] (4,-4,-1) -- (4,-4,1);
\draw[thick] (-4,-4,-1) -- (-4,-4,1);
\draw[thick, dashed] (-4,4,-1) -- (-4,-4,-1);
\draw[thick, dashed] (4,4,-1) -- (-4,4,-1);
\draw[thick] (4,-4,-1) -- (4,4,-1);
\draw[thick] (4,-4,-1) -- (-4,-4,-1);
\path (-3.5,0,-2.5) node {$A \B \theta^u$};
\path (3.5,0,-2.5) node {$B \B \theta^u$};
\path (-6,0,0) node {$\row(A) \cong \R^a$};
\path (6,0,0) node {$\row(B) \cong \R^b$};
\path (-6.5,0,-2.5) node {$\col(A) \cong \R^a$};
\path (-6.7,0,-2.9) node {\small e.g. Grocery Recs};
\path (6.5,0,-2.5) node {$\col(B) \cong \R^b$};
\path (7.2,0,-2.9) node {\small e.g. Restaurant Recs};
\filldraw[
        draw=blue,%
        fill=blue!40,
        opacity=0.7,%
    ] (-.75,-3,-.5) -- (-.75,3,-.5) -- (3.25,3,.5) -- (3.25,-3,.5) -- cycle;

\filldraw[
        draw=red,%
        fill=red!40,%
        opacity=0.7,
    ] (-3.25,-3,.5) -- (-3.25,3,.5) -- (.75,3,-.5) -- (.75,-3,-.5) -- cycle;
    
\filldraw[
        draw=blue,%
        fill=blue!40,
        opacity=0.7,%
    ] (1,-3,-3.5) -- (1,3,-3.5) -- (5,3,-3.5) -- (5,-3,-3.5) -- cycle;
    
\filldraw[
        draw=red,%
        fill=red!40,%
        opacity=0.7,
    ] (-5,-3,-3.5) -- (-5,3,-3.5) -- (-1,3,-3.5) -- (-1,-3,-3.5) -- cycle;
\end{scope}
\end{tikzpicture}}
    \caption{\textbf{Visualizing source and target tasks.} In our model, users' personal data $\B \theta^u$ lie in a latent space $\R^n$. For a given source matrix $A$ and target matrix $B$, the image $A \B \theta^u \in \R^a$ gives the reward parameter for the source task, and $B \B \theta^u \in \R^b$ gives the reward parameter for the target task. Thus, $\row(A) \subseteq \R^n$ and $\row(B) \subseteq \R^n$ can be thought of as the corresponding partitions of the full personal data that are used by source and target tasks respectively.  If the source and target tasks make use of similar segments of the users' personal data---that is, if $\row(A) \cap \row(B)$ has high dimension---then, we may construct a low rank model (boxed) to represent both the source and target tasks. In this case, having knowledge of $A \B \theta^u$ may help in recovering $B \B \theta^u$.}
    \label{fig:projection}
\end{figure}
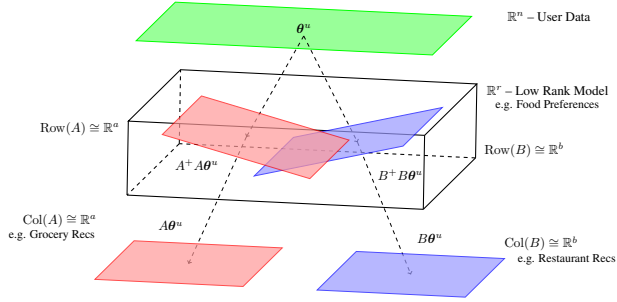

We now consider a multi-armed bandit problem with finite horizon $H$. At each time $h \in [H]$, the agent observes a \textit{context} consisting of a user-task pair $(u_h,t_h) \in [U] \times [T]$, and subsequently chooses an arm $\B x$ from a finite and time-dependent decision set $\mc A_h$. We then observe the reward
\begin{equation}\label{eq:reward}
    X_h = \cyc{\B x , B_{t_h} \B \theta^{u_h}} + \eta_h,
\end{equation}
where $\{\eta_h\}_{h=1}^{H}$ are i.i.d., mean-$0$, independent $1$-subgaussian random variables. We assume that the decision sets are well-conditioned, such that for any arm $\B x \in \mc A_h$, $\norm{B_{t} \B \theta^{u}} \le 1$ for all $t,u$. We also assume $\max_{h \in [H], \B x \in \mc A_h} \norm{\B x}_2 \le 1$, so each arm has $2$-norm at most $1$, and that $\norm{\B \theta^{u}} \le 1$ for all~$u$.

Crucially, in our setup, the agent is \textit{not} given access to any of the personal data vectors $\B \theta^u$, or any of the source/task matrices $A$ or $B_t$. The agent is, however, given access to each of the optimal arms $\B \alpha^u$ and corresponding reward $\beta^u$ for an agent that solved the source task. These quantities can be used to approximate the source reward parameter $A \B \theta^u$ under reasonable assumptions on the source agent's arm space. Namely, we assume that the source agent's arm space is a net of some ball of radius $S$, whence we may write $\B \alpha^u \approx S\frac{A \B \theta^u}{\norm{A \B \theta^u}_2}$. The corresponding expected reward, as per (\ref{eq:reward}), is then approximately given by $\beta^u \approx S\norm{A \B \theta^u}_2$. We may then approximate $A \B \theta^u \approx \frac{\beta^u \B \alpha^u}{\norm{\B \alpha^u}_2^2}$; in the remainder of this paper we assume, for simplicity, that this approximation is exact.    

As usual, the objective of the agent is to maximize the expected cumulative reward $\sum_{h=1}^H\ex{X_h}$, or equivalently minimize the expected regret
\begin{equation}
    R_H := \ex{\sum_{t \in [T]} \sum_{u \in [U]} \sum_{h : (u_h, t_h) = (u,t)} \max_{\B x \in \mc A_h}\cyc{\B x, B_t\B \theta^u} - \sum_{h=1}^H X_h}.
\end{equation}
 Our central contribution in this paper is demonstrating how knowledge of the optimal arms for the each of the source tasks may be used to significantly reduce regret relative to the naive approach solving each of the $U \cdot T$ contexts separately. Using the popular LinUCB algorithm \cite{li2010contextual}, we obtain the following baseline regret bound. (All proofs are in \Cref{app:proofs}.)

\begin{theorem}[LinUCB Bound]\label{thm:linucbtrivial}
A LinUCB agent that naively solves each of the $U\cdot T$ tasks separately has regret bounded by
\[R_H \le O\pa{b \sqrt{HTU} \log\pa{\frac{H}{TU}}}. \]
\end{theorem}

An interpretation of \cref{thm:linucbtrivial} is that on average, each user receives an arm with reward at most $ R_H / H = b\sqrt{\frac{TU}{H}} \log\pa{\frac{H}{TU}}$ worse than the best, meaning that the problem is harder when there are more tasks, more users, and higher dimension, but becomes easier with time.

\subsection{Related Work}\label{sec:prelims-comparisons}

\paragraph{Recommendation Systems}
Recommendation systems---which recommend items to users, given data on users, items, and/or their pairwise interactions---have been one of the most impactful applications of modern machine learning \cite{bobadilla2013recommender}. Moreover, the study of recommendation systems  has led to a number of foundational methodological breakthroughs \cite{candes2009exact, koren2009matrix}.

\paragraph{Contextual Bandits} Our framework is an example of a setting with \textit{contextual bandits} \cite{abe2003reinforcement}, i.e., one in which bandits are allowed to draw upon information beyond the arms themselves---the \textit{context}---to help make their decisions. Moreover, our framework follows that of a linear bandit. However, our framework differs from the standard model  of contextual linear bandits \cite[Equation 19.1]{lattimore2020bandit}, which represents the rewards as  $ X_h = \langle \psi(c, \B x), \B \theta^{\ast} \rangle, $
where $\psi$ is a known feature map and $\B \theta^{\ast}$ is an unknown reward parameter. While our question can be fit into that standard framework by using \begin{gather*} \psi((u, t), \B x) = \begin{bmatrix} 0 & \ldots &  \underbrace{\B x^{\top}}_{(u, t)^{\text{th}} \text{ vector }} & \ldots & 0 \end{bmatrix}^{\top}\quad\text{and} \\ \B \theta^{\ast} = \begin{bmatrix} (B_1\B \theta^1)^{\top} & \ldots & (B_t\B \theta^u)^{\top} & \ldots & (B_T\B \theta^U)^{\top} \end{bmatrix}^{\top}, 
\end{gather*} such a setup does not allow for us to easily incorporate our information about $A\B \theta^u$ into the problem, as it merely views $\B \theta^{\ast}$ as an unrolling of a vector. In particular, the direct bounds obtained through the standard contextual bandits formulation (without further assumptions) are weak relative to ours, as $\B \theta^{\ast}$ has dimension $bTU$. 

\paragraph{Explore and Refine Approaches}
Our algorithm in \Cref{sec:model2} is methodologically similar to the algorithm of \cite{jun2019bilinear}, which also explores to obtain a subspace containing the solution, and imposes a soft penalty to make the estimated reward parameter close to that subspace. Indeed, our approach draws upon their LowOFUL algorithm.  %
However, the problem we seek to solve is quite different from that of \cite{jun2019bilinear}---we have either multiple tasks or side information in each of our models.

\paragraph{Cross-Domain Recommendations}
Cross-Domain recommendations (see \cite{zhu2021cross} for a survey) is an area where data from other recommendation domains is used to inform recommendations in a new domain. Here we highlight works in cross-domain recommendations that we find particularly relevant. 

\cite{hu2013personalized} makes a low-rank assumption on the tensor of item-user-domain, which is similar to but differs from our assumption, as our assumption does not make any assumptions regarding the quantity of items, and hence about rank constraints along that axis. Furthermore, those works recommend items, whereas our work chooses from a set of arms---the featurizations in our model are provided ahead of time, whereas in \cite{hu2013personalized} they are learnt.

\cite{winoto2008if} uses a similar cross-domain recommendation idea; but they use explicit relationships between items as hinted by the title (i.e. aficionados of the book "The Devil Wears Prada" should also enjoy the movie "The Devil Wears Prada"), which our work does not depend on. Our work makes purely low-rank based assumptions, making it more general than \cite{winoto2008if}.

\paragraph{Meta-Learning}

In addition, our setup resembles the standard meta-learning \cite{schmidhuber1987evolutionary, bengio1990learning} setup, as we have several interrelated tasks (namely, each (service, user) pair can be viewed as a separate task). In \cite{cella2020meta}, the authors address a meta-learning task by assuming that the task matrices come from a simple distribution, which are then learnt. One way to think of this work is that it corresponds to ours in the case where $U = 1$. As such, it cannot share parameters across users, and thus must make distributional assumptions on the tasks to get nontrivial guarantees. 

\paragraph{Transfer Learning and Domain Adaptation} 
Our setup also resembles that of transfer learning, as we use parameters from one learned task directly in another task, as well as domain adaptation, because we use similar models across two different settings. Transfer learning has commonly used in the bandit setting \cite{azar2013sequential, calandriello2014sparse, zhang2017transfer, deshmukh2017multi}.

Most relevant to our work, the model we present in \Cref{sec:model2} is similar to the setting of \cite{liu2018transferable}, in which the target features are a linear transformation away from the source features, and the bandit has access to historical trajectories of the source task. One way to interpret this in our model is that the work assumes knowledge of the full set of features $\B \theta$ (as opposed to merely a projection), and thus strives to solve for $B$. It then solves for $B$ by applying linear regression to align the contexts, obtaining a nontrivial regret guarantee. However, the result is not directly comparable, because the model contains substantially more information than ours, namely access to the full $\B \theta$ instead of simply a view $A \B \theta$ as well as past trajectories on a bandit problem.

The paper \cite{kanagawa2019cross} uses domain adaptation to approach cross-domain recommendation and shows that it can achieve superior performance to certain baselines. However, it is not clear how to implement this in a bandit-based framework.

\paragraph{User-similarity for recommendation}  In works like \cite{cesa2013gang, yang2020laplacian, soare2014multi, gentile2014online}, the authors address recommendation problems by using information about similar users, like users in a social network. They assume that people who are connected on social media have similar parameter vectors. However, this assumption may potentially violate privacy by assuming access to a user's contacts, and by using many tasks (as these works effectively have $T = 1$) and many users we can obtain similar results without using social data.

\pdfoutput=1

\section{One Target, One Source, Many Users}\label{sec:model2}
In this section, we consider the case where we have $T = 1$ target task. We show that source recommendations can be used to generate target recommendations in two phases. We first show how an agent with full knowledge of the source and target matrices can reduce regret using the source recommendations by reducing the dimensionality of target recommendation problem. We then show that agent without knowledge of the source and target matrices can still reduce regret under a realistic assumption on the ordering of the contexts by \textit{learning} the underlying relationship between the source and target matrices from observed data. One quantity that relates to the success of our approach is the \textit{common information}, which we define as 
\begin{equation}
    \kappa := \frac{\dim \pa{\row(A) \cap \row(B)}}{b}.
\end{equation}

This quantity can be thought of as the fraction of information about the underlying reward parameter $B \B \theta$ that can be possibly be explained given the source information $A \B \theta$. Referring back to \Cref{fig:projection}, we can see that if $\kappa \approx 1$, then the target task uses segments of the user's information that are almost all contained by the segments used by the source task  -- a situation that arises often in practice. Stated differently, this implies that the source and target matrices $A$ and $B$ have a low rank structure. To see this, observe that the rank of the block matrix $r := \operatorname{rank}\pa{\bmat{A \\ B}}$ is given by
\[r = a + b - \dim (\row(A) \cap \row(B)) = a + (1 - \kappa)b.\]

In our running example, the matrix $A$ represents a grocery recommendation task and the matrix $B$ represents a restaurant recommendation task. Our algorithm finds the minimal set of information needed to predict both tasks, and then predicts the user's restaurant preferences by looking at all possibilities that are consistent with the user's grocery preferences. 

\paragraph{Skyline Regret Bound}\label{sec:model2-skyline} 

We now show how an agent that was given the source and target matrices $A$ and $B$ may use this structure to reduce regret. By the above, we may write the compact singular value decomposition of the stacked matrix as $\bmat{A \\ B} = \mathcal{U} \Sigma \mathcal{V}^T$, where $\mathcal{U}, \mc{V} \in \R^{(a+b) \times r}$ and $\Sigma \in \R^{r \times r}$. Projecting onto the principal components, we can create a low-dimensional generative model
\[\bmat{A \B \theta^u \\ B \B \theta^u} = \mc{U} \pa{\Sigma \mc V^T \B \theta^u} := \mc{U} \B \phi^u,\]
where $\B \phi^u \in \R^r$. Letting $\pi_A \in \R^{a \times(a + b)}$ denote the orthogonal projector onto the first $a$ components and $\pi_B \in \R^{b \times(a + b)}$ denote the orthogonal projector onto the last $b$ components, we have that $\pi_A \mc{U} \B \phi^u = A \B \theta^u$, whence we may write $\B \phi^u \in \pa{\pi_A \mc{U}}^+ A \B \theta^u + \nul(A)$.
Here, $\pa{\pi_A \mc{U}}^+$ denotes the Moore-Penrose pseudo-inverse. Substituting in, we may write
\begin{equation}\label{eq:decomp}
B \B \theta^u = \pi_B \mc{U} \B \phi^u \in \pi_B \mc{U} \pa{\pi_A \mc{U}}^+ A \B \theta^u + \pi_B \mc{U} \nul(A)
.\end{equation}
We call $D_T := \pi_B \mc{U} \B \phi^u \in \pi_B \mc{U} \pa{\pi_A \mc{U}}^+ \in \R^{b \times a}$ the \textit{transformer}, as it transforms the known reward parameter from the source domain into the target domain. From the above, we have that the residual $B \B \theta^u - D_TA \B \theta^u$ lies in $\pi_B \mc{U} \nul(A)$, which is a space of dimension $(1 - \kappa)b$. Choosing an orthogonal basis for this space, it follows that we may write
\begin{equation}\label{eq:normal}
    B \B \theta^u = \underbrace{D_T A \B \theta^u}_{\text{Systematic Component}} + \underbrace{D_G \B \psi^u}_{\text{Idiosyncratic Component}},
\end{equation}
where $D_G \in \R^{b \times (1 - \kappa)b}$ has orthogonal columns given by the basis elements, and $\B \psi^u \in \R^{(1 - \kappa)b}$ gives the corresponding weights. We call $D_G$ the \textit{generator}, as it generates the second term in the above expression from a lower dimensional vector $\B \psi^u$. Together $D_T$ and $D_G$ give a useful decomposition of the target reward parameter into a \textit{systematic component} that can be inferred from the source reward parameter, and an \textit{idiosyncratic component} that must be learned independently for each user. An agent with full knowledge of the source and task matrices can compute $D_G$ and $D_T$ as given above, and thus needs only solve for each of the $\B \psi^u$ that lie in a $(1-\kappa)b$-dimensional subspace. As a result, it attains a lower skyline regret bound.

\begin{theorem}[Skyline Regret]\label{lem:model2skyline}
An agent with full knowledge of the source and target matrices $A$ and $B$ may attain regret
\[R_H \le O\pa{(1 - \kappa)b \sqrt{H} \log\pa{\frac{H}{U}}}. \]
\end{theorem}

Comparing to Theorem \ref{thm:linucbtrivial}, we see that we obtain an identical bound, but with a factor of $1-\kappa$ in front. Thus, as alluded to earlier, the improvement depends on how large $\kappa$ is. If $\kappa$ is close to one, then we obtain a significant improvement, corresponding to being able to use a substantial amount of information from the source recommendations. This result only applies when the agent is given the source and target matrices. 

\paragraph{Learning from Observed Data}

We now return to our original setup, where the source and target matrices are unknown. We show how the transformer $D_T$ and the generator $D_G$ may be learned from observed data under a realistic assumption on the ordering of the contexts. We assume that some subset of the users $U_0$ belongs to a \textit{beta group}, and are the only users requesting recommendations for the first $H_0$ iterations. After the first $H_0$ iterations, all users are allowed to request recommendations for the target task. This assumption makes learning the transformer and generator easier as it reduces the number of idiosyncratic parameters we must fully learn in order to understand the relationship between the two tasks.

Our approach is briefly summarized as follows. In the first $H_0$ iterations, we use some \textit{exploration policy} that independently pulls arms for each user. In practice, this policy could be adaptive, like the baseline LinUCB algorithm. To keep our analysis tractable, we instead assume that the exploration policy is \textit{oblivious}. That is, we assume that arms pulled in the first $H_0$ iterations do not depend on any observed rewards, and that the arms are randomly drawn from $\frac{1}{\sqrt{b}}\mc{N}(0,I_b)$. We further assume that the same oblivious policy is used across users. This is essentially equivalent to using different oblivious strategies for different users, as we have no prior information on the reward parameters. In our experiments, we show that this obliviousness assumption does not appear necessary to achieve good performance. We find that adaptive policies (namely LinUCB) are just as, if not more effective at learning the model relating the two tasks, while achieving performance equivalent to the baseline during the exploration phase.

Immediately after iteration $H_0$, we run least-squares for each user on the arm-reward history we have observed thus far to obtain estimates $\widehat{B \B \theta^u}$ for each user $u$. Observe by \Cref{eq:decomp} that the transformer $D_T$ is determined by the left singular vectors $U$. As $\nul(A) = \nul(\pi_A \mc{U})$, $D_G$ may also be computed given $U$. We compute the singular value decomposition

\begin{equation}\label{eq:stack}
     \bmat{A \B \theta^1 &  \hdots & A \B \theta^{U_0}\\ \widehat {B \B \theta^1} & \hdots & \widehat{ B \B \theta^{U_0}}} = \bmat{\widehat{\mc{U}} & \widetilde{\mc{U}}} \bmat{\Sigma_r & 0 \\ 0 & \Sigma_{\kappa b}} \bmat{\widehat{\mc{V}^T} \\ \widetilde{\mc{V}^T}};
\end{equation}
we then use $\widehat{\mc{U}}$, which corresponds to the first $r$ columns of the decomposition, as a proxy for $U$.

Note that in the skyline, we showed that $B\B \theta^u$ lies in the subspace $
    S_u := \set{\pi_B \mc{U} \B \phi \mid \pi_A \mc{U} \B \phi = A \B \theta^u, \B \phi \in \R^r}.
$

By substituting $\widehat{\mc{U}}$ for $\mc{U}$, and subsequently computing $\widehat D_T$ and $\widehat D_G$ as in \Cref{eq:decomp}, we obtain the approximation to $S_u$ given by
\begin{align*}
    \widehat S_u &:= \set{\pi_B \widehat{\mc{U}} \B \phi \mid \pi_A \widehat{\mc{U}} \B \phi = A \B \theta^u, \B \phi \in \R^r}\\
    &\subset \operatorname{span}\pa{\widehat D_T A \B \theta^u, \widehat D_G} =: E_u.
\end{align*}

Although $\widehat S_u$ is an affine space, it is a subset of the linear space $E_u$. If $B \B \theta^u$ is close to the space $E_u$ (i.e. if our approximation is close) then in the coordinates of some orthogonal basis $W := \bmat{W_{\parallel} & W_\perp}$ for $\R^b$, where $W_{\parallel}$ is a basis for $E_u$, $B \B \theta^u$ is approximately sparse (i.e. components of $B \B \theta^u$ that lie outside of $E_u$ have small magnitude). We invoke an algorithm known as LowOFUL \cite{jun2019bilinear}, which works in precisely this setting to select arms for all iterations $h \in [H_0 + 1, H]$. We give a full description of our algorithm in Algorithm~\ref{alg:model2}. More concretely, the regret for our approach may be decomposed as
\begin{align*}
    R_H &= \ex{\sum_{u \in [U]} \sum_{h \le  H_0 : u_h = u} \pa{\max_{\B x \in \mc A_h} \cyc{\B x, B\B \theta^u} - X_h}}\\
    & \quad + \sum_{u \in [U]} \ex{\sum_{h > H_0 : u_h = u} \pa{\max_{\B x \in W^T \mc A_h} \cyc{\B x,  W^T B\B \theta^u} - X_h}}\\
    &=: \mathrm{Explore}_{H_0} + \sum_{u \in [U]} \mathrm{LowOFUL}^u_{H_u}
\end{align*}
where $H_u$ is the number of times user $u$ requests a recommendation. Our approach achieves regret comparable to the skyline agent that has knowledge of the source and target matrices. Each term $\mathrm{LowOFUL}^u_{H_u}$ is bounded in terms of the magnitude of the ``almost sparse'' components $L^u := \norm{(W^TB \B \theta^u)_{(1-\kappa)b+1:b}}_2$. With methods from perturbation theory, we show:
\begin{lemma}\label{lem:closetospace} With probability $1 - \delta$, we have that for all $u$,
\[ L^u \le  \sqrt{2} \pa{1 + \pa{\sigma_{\mathrm{min}}(\pi_A \mc{U}) - \sqrt{2}Z}^{-1}}Z , \] where $q = \frac{\sqrt{bU_0(b + \log(2/\delta))}}{\sqrt{H_0 / U_0} - C(\sqrt{b} + \sqrt{\log(4 / \delta)})} \approx \frac{bU_0}{\sqrt{H_0} - C\sqrt{U_0b}}$; $Z = \frac{q}{\sigma_{\mathrm{min}}\pa{M} - q}$; $M = \tiny{\bmat{A \B \theta^1 &  \hdots & A \B \theta^{U_0}\\ B \B \theta^1 & \hdots &  B \B \theta^{U_0}}}$; $C$ is some absolute constant; and we assume that $\sigma_{\mathrm{min}}\pa{M} > q$.
\end{lemma}

We now apply the original LowOFUL \cite[Corollary 1]{jun2019bilinear} regret bound to show:

\begin{theorem}\label{thm:giga} With an oblivious exploration policy that draws identical arms for each user from $\frac{1}{\sqrt{b}}\mc{N}(0,I_b)$,   assuming that each user appears equally many times, with probability at least $1 - \pa{U\delta' + \delta}$, we have
\begin{align*}
    &\sum_{u \in [U]} \mathrm{LowOFUL}^u_{H_u} 
    \le\\
    &\quad \quad \quad  O\biggr(\underbrace{(1-\kappa)b \sqrt{UH}\log\pa{\frac{H}{U}}\sqrt{\log (1 / \delta')}}_{\text{Skyline Regret (adjusted by $\delta'$)}}\\
    & \quad \quad \quad + \underbrace{H \pa{1 + \pa{\sigma_{\mathrm{min}}(\pi_A \mc{U}) - \sqrt{2}Z}^{-1}}Z \log\pa{\frac{H}{U}}}_{\text{Approximation Error from Using $\widehat{S}_u$}} \biggr ).
\end{align*}
\end{theorem}

Observe that the first sum term in this bound essentially matches the skyline bound (\Cref{lem:model2skyline}), and has regret scaled by $1-\kappa$. We pick up an additional penalty term given by the second summand that comes from the approximation error in $\mc{U}$. Looking more closely at the term $Z$, our bound predicts that we obtain more regret when the number of users in the beta group $U_0$ increases as we have to learn more idiosyncratic components, justifying the user growth assumption, and less regret as the time $H_0$ we have to learn the idiosyncratic components increases,  Finally, when $H_0 = cH$ for some constant $c$ and $U \rightarrow \infty$ while holding $U_0$ fixed, the second term is $o(1)$ times the first term.

\pdfoutput=1

\section{No Sources, Many Targets, Many Users}\label{sec:model3}

The previous section showed that we can obtain learning guarantees from a source. In this section, we address that case where we have no sources but many targets -- the reward at time $h$ is given by $X_h = \cyc{\B x, B_{t_h} \B \theta^{u_h}} + \eta_h$, where $t$ and $u$ index the task and user respectively. As before, we compare to a LinUCB baseline (\Cref{thm:linucbtrivial}) that uses one bandit for each of the $U \cdot T$ contexts. In the previous section, we found that source tasks with low rank $r$ (or equivalently high common information $\kappa$) were effective in producing recommendations for the target task. We similarly assume in this setting that there exist user-specific vectors $\B \phi^u \in \R^r$ of dimension $r \ll n$ along with matrices $\mc U_1, \dots,\mc U_T$ such that $B_t \B \theta^u = \mc U_t \B \phi^u$ for all $u$ and $t$. 
 
In our running example, the matrices $B_t$ represent many different food delivery apps, perhaps focused on different  food sources. Our low-rank assumption posits that there is substantial redundancy between the different tasks (in the example: user-specific food preferences $\B \phi^u$ are consistent across apps). One benefit is that after many users has been solved for, solving new tasks is very easy because we have data on a large set of users. 

\paragraph{Skyline Regret Bound}

In the skyline, we assume full access to the task matrices -- we show that in this case, we obtain a reduction to $U$ separate dimension-$r$ bandit problems, one for each user.

\begin{lemma}[Skyline Bound]\label{lem:model3skyline}
With access to all task matrices, we obtain a regret of \[ R_H \le O\pa{r\sqrt{UH}\log\pa{\frac{H}{U}}} .\]
\end{lemma}

\paragraph{An Algorithm based on User and Task Growth}

Our algorithm in the previous section required the existence of a beta group of users to learn the relationship between the source and target. In this multitask setting, we additionally require a a beta group of tasks. We first apply an exploration strategy as in Section~\ref{sec:model2} to make estimates $\widehat{B_t\B \theta^{u}}$ for a small number of users $U_0$ and tasks $T_0$. We subsequently compute a singular value decomposition, as in \Cref{eq:stack}, to produce estimates $\widehat{\mc U_t}$ for $t \in [T_0]$. Next, we let all $U$ users request recommendations on the $T_0$ tasks, and use the $\widehat{\mc U_t}$ to make estimates for the low rank latent parameters $\widehat{\B \phi^u}$ for all $u \in [U]$. We accomplish this by running least-squares on the arm-reward history, noting that for any arm $\B x$,
\[\cyc{\B x, B_t \B \theta^u} = \cyc{\B x, \mc U_t \B \phi^u} = \cyc{\pa{\mc U_t}^T \B x, \B \phi^u} \approx \cyc{\pa{\widehat{\mc U_t}}^T\B x, \B \phi^u}.\]

In the last phase, we let all $U$ users request recommendations on all $T$ tasks, and again use these approximations, in conjunction with the reward history, to learn all $\widehat{\mc U_t}$ for all $t \in [T]$ by least squares
\begin{align*}
    \cyc{\B x, B_t \B \theta^u} = \cyc{\B x, \mc U_t \B \phi^u} &= \cyc{\B x \pa{\B \phi^u}^\top, \operatorname{vec}\pa{\mc U_t}} \approx \cyc{\B x \pa{\widehat{\B \phi^u}}^\top, \operatorname{vec}\pa{\mc U_t}}.
\end{align*}

Observe that the sample complexity of learning the $\widehat{\mc U_t}$ is greatly reduced as we make use of \textit{all} estimates $\widehat{\B \phi^u}$ to estimate each $\widehat{\mc U_t}$. That is, no matter which user $u_h$ appears for a given task context $t_h$, we are able to collect some sample that helps in approximating $\widehat{\mc U_{t_h}}$, as we have access to a quality estimate of the low rank latent parameter $\widehat{\B \phi^{u_h}}$ for the user $u_h$. Although we do not give a theoretical analysis of this algorithm (for barriers described in \Cref{app:hard}) we show in synthetic experiments that our approach offers a clear advantage over the baseline approach of solving each of the $U \cdot T$ contexts separately. We give a full description of the algorithm in Algorithm~\ref{alg:model3}.

\pdfoutput=1

\section{Simulation Experiments}\label{sec:exp}

In this section, we present several simulation experiments to benchmark our approach.\footnote{
Ideally, we would also perform an experimental evaluation on real-world data. To accomplish this, however, we would need to obtain data on user behavior for the same collection of users over multiple services. The closest matches to this requirement are datasets such as Douban \cite{zhu2020deep} that show the action-reward history for a real cross-domain recommendation service. Unfortunately, although the Douban dataset does record user behavior over multiple services, the only actions that we know rewards for are those that were executed by the service. That is, each user only gave feedback for items recommended to them by the service, which in practice is a very small fraction of the total action space. Thus, this dataset cannot be used to assess our algorithms, since the recommendations our algorithms would make would  not coincide with those that received feedback in Douban.

} We tested \textbf{Rec2}'s regret, using a LinUCB implementation following that of \cite{wang2017factorization}.\footnote{\label{ssc:expdet} Each run of our algorithms takes at most five seconds on a 8-core 2.3GHz Intel Core i9 processor.} %
We provide the full experimental setup in \Cref{app:runtime}, only highlighting the most important details here.

\paragraph{Experiment with One Source, One Target, Many Users}\label{ssc:exp1}

In this experiment, we choose $\kappa = 0.9$, $a = 20$, $b = 20$, $U_0 = 25$, $U = 500$, and $|\mc{A}| = 40$, $H_0 = 2,000$ and $H = 8,000$. We generate $A$ and $B$ by letting $\B \phi^{u}$ be random $d$-dimensional vectors, where $d = a + b(1-\kappa)$, and $A$, $B$ are random $a \times d$ and $b \times d$-dimensional matrices.  As we can see, the results of our \textbf{Rec2} approaches greatly improve upon the baselines. Notably, while we only have bounds on the excess regret after exploration for the oblivious exploration policy, it appears that similar bounds hold for both the oblivious exploration policy and the LinUCB exploration policy. 

\paragraph{Experiment with No Sources, Many Targets and Many Users}\label{ssc:exp2}

In this experiment, we choose $b = 3$, $T = 30$, $U = 50$, $r = 6$, $|\mc{A}| = 40$, and $T_0 = 3$, $U_0 = 5$. We choose random vectors $\B \phi^u \in \R^{r}$, and each task matrix $B_i$ is a $b \times r$ matrix. Actions are random vectors of length $b$. We run Algorithm~\ref{alg:model3} for $H_0 = 1,000$ timesteps in the first phase, $H_1 - H_0 = 3,000$ timesteps in the second phase, and $H - H_1 = 9,000$ timesteps in the third phase and see that the results of Rec2 strongly outperform the baseline in the third phase.

\begin{figure}[t!]
    \centering
    \subfloat[\centering Results for one source, one target ]{{\includegraphics[width=6cm]{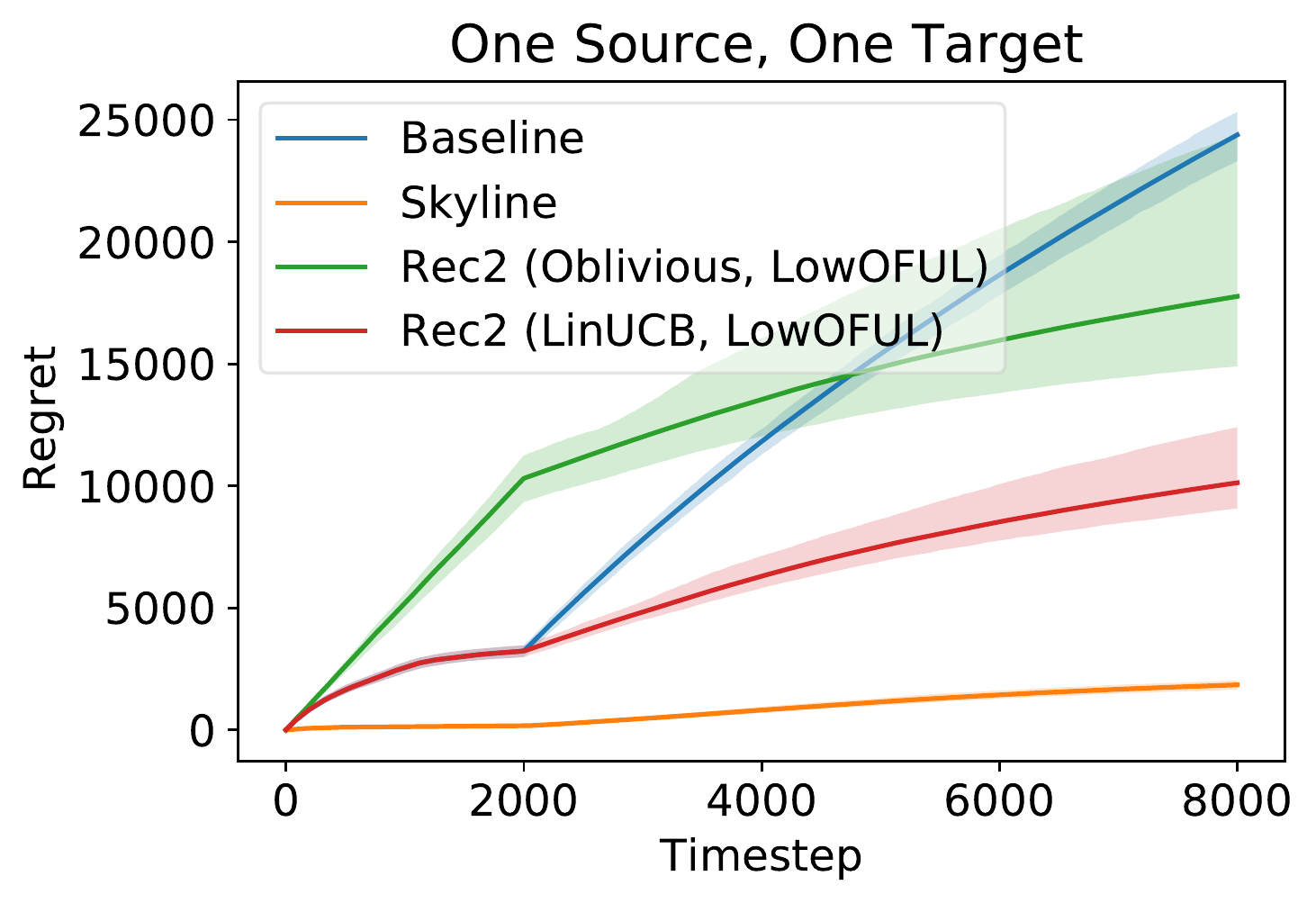} }}%
    \hfill
    \subfloat[\centering Results for no sources, many targets ]{{\includegraphics[width=6cm]{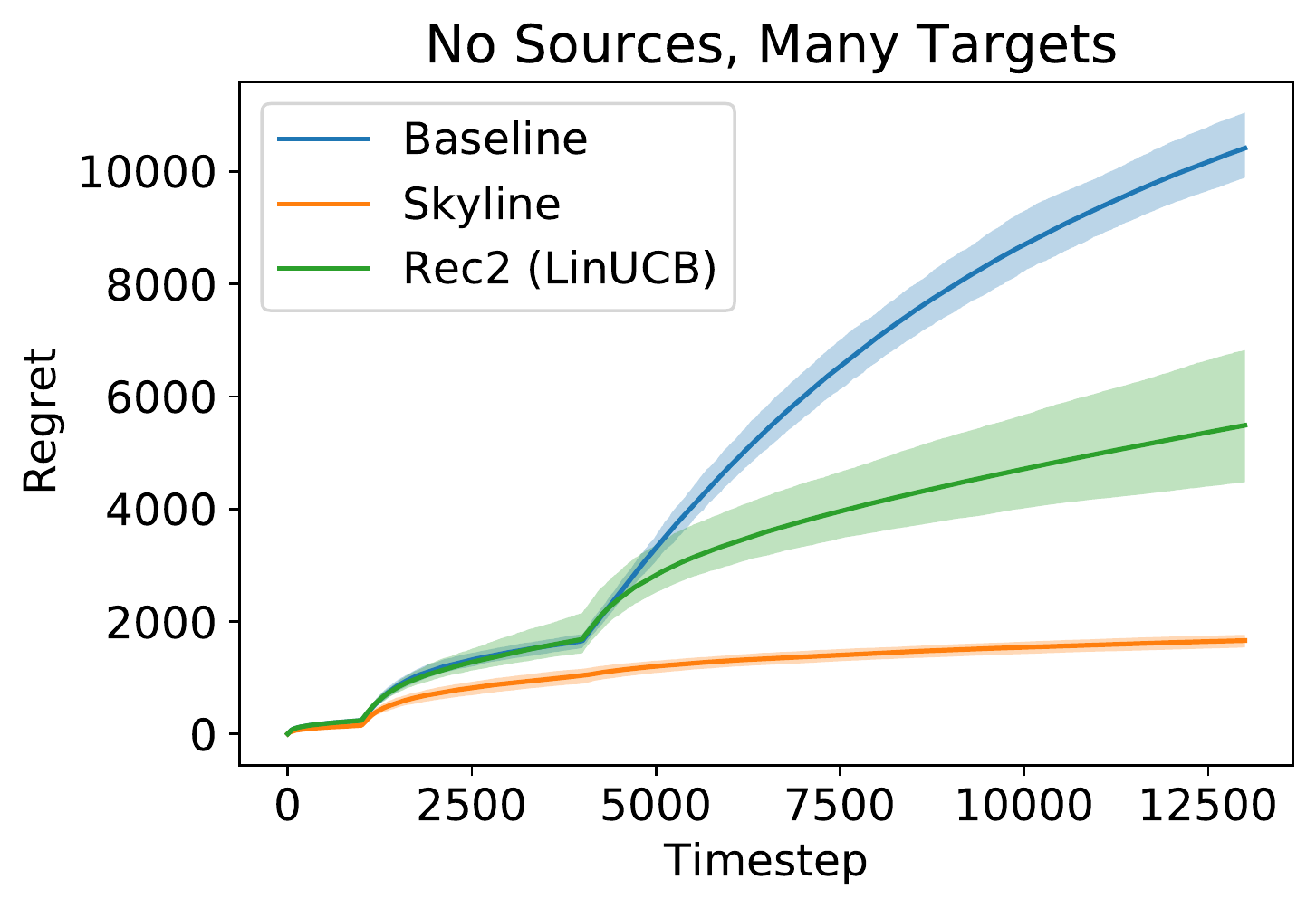} }}%
    \caption{\textbf{Average Regret, with 95\% confidence intervals shown. } Note that in (a) the post-exploration phase (after $H_0 = 2000$) in the combination of Oblivious and LowOFUL sharply bends downwards, as predicted by our theorem, yet somehow the LinUCB exploration policy appears to incur equal or even less regret after the exploration phase without sacrificing any regret in the exploration phase. In (b) it appears that while the \textbf{Rec2}  approach does not obtain lower regret in the first and second phases, the regret becomes sharply lower in the third phase.}
    \label{fig:mainfig}
\end{figure}

\pdfoutput=1

\section{Conclusion}\label{sec:concl}

We introduced a contextual multi-armed bandit recommendation framework under which existing recommendations   serve as input for new recommendation tasks. We show how to decompose the target problem into systematic and idiosyncratic components. The systematic component can be extrapolated from the source recommendations;  the idiosyncratic component must be learned directly but has significantly lower dimensionality than the original task. The decomposition itself can be learned and utilized by an algorithm that that executes a variant of LinUCB in a low-dimensional subspace, achieving regret comparable to a strong skyline bound. Our approach extends to a setting with simultaneous targets but no source, and the  methods perform well in simulation.

We believe that \emph{recommending with recommendations}  should make it possible to reduce recommendation systems' reliance on sensitive user data. Indeed, there are already large corpora of recommendations (in the form of advertisements) produced by search engines and social media platforms that have access to tremendous amounts of user information; these span a wide range of domains and thus could be used as source input in many different target tasks.\footnote{Ironically, in this sense, advertisements effectively act as a privacy filter -- they can actually serve as a fairly complete profile of a person, while abstracting from much of the underlying sensitive data.} Moreover, individual platforms such as grocery delivery services  are developing expertise in recommending products within specific domains; these could naturally be extrapolated to target tasks with some degree of overlap, such as the restaurant recommendation service in our running example.

Operationalizing the approach we provide here just requires a way for users to share the recommendations from one service with another. There are already public data sets for this in advertising domains, but in the future it could be accomplished more robustly either through a third-party tool that users can use to scrape and locally store their recommendations, or through direct API integrations with existing recommendation engines. But in any event, once a given user's existing recommendations are available, they can be loaded into a new service (with the user's consent) and then can instantaneously learn the systematic component and provide valuable recommendations using the approach we described here.

In addition to the overall \textbf{privacy} benefits of basing recommendations on existing recommendation sources rather than the underlying user data, our method \textbf{lowers barriers to entry} by creating a way for new services to bootstrap off of  other platforms' recommendations. Moreover, this approach to some extent \textbf{empowers users} by giving them a way to derive additional value from the recommendations that platforms have already generated from their data. %

\bibliographystyle{unsrt}
\bibliography{refs}
\appendix
\onecolumn
\section*{Supplementary Material}
\pdfoutput=1

\section{Proofs Omitted from the Main Text}\label{app:proofs}

\subsection{Proof of \Cref{thm:linucbtrivial}}
\label{proof:linucbtrivial}

\begin{theorem*}[LinUCB Bound]
A LinUCB agent that naively solves each of the $U\cdot T$ tasks separately has regret bounded by
\[R_H \le O\pa{b \sqrt{HTU} \log\pa{\frac{H}{TU}}}. \]
\end{theorem*}

We first restate the LinUCB regret bound. 

\begin{lemma}[{ \cite[Corollary 19.3]{lattimore2020bandit}}]
\label{lem:linucb}
Suppose that a LinUCB agent aims to solve a stochastic linear bandit problem with reward parameter $\B \theta$. If the agent's decision sets $\mc A_h$ for $h \in [H]$ satisfy
\[\max_{h \in [H]} \max_{\B x, \B y \in \mc A_h} \cyc{\B x - \B y, \B \theta} \le 1, \qquad \qquad \max_{h \in [H]} \max_{\B x \in \mc A_h} \norm{\B x}_2 \le 1,\]
then the regret at time $h$ is bounded by
\[R_h \le O\pa{b \sqrt{h} \log \pa{h}}.\]
\end{lemma}

In our setup, there are $U\cdot T$ contexts, corresponding to each user-task pair. We define $R_h^{u,t}$ to be the regret of a LinUCB agent that runs for $h$ iterations for the reward parameter $B_t \B \theta^u$. We further define 
\[H_t^u = \sum_{h = 1}^H \mathbbm{1}\bra{(u_h, t_h) = (u,t)}\]
to be the number of times that the agent encountered the user-task pair $(u,t)$. Observe that by our assumptions on the decision set in \Cref{sec:prelims-model}, the assumptions of \Cref{lem:linucb} have been met. As the LinUCB algorithm is anytime (i.e. it does not require advance knowledge of the horizon), we can decompose the regret as
\begin{align*}
    R_H = \ex{\sum_{u \in [U], t \in [T]} R_{H_t^u}^{u,t}} &\le \ex{b\sum_{u \in [U], t \in [T]} O\pa{\sqrt{H_t^u} \log \pa{H_t^u}}}\\
    &\le \sum_{u \in [U], t \in [T]} O\pa{b\sqrt{\frac{H}{UT}} \log \pa{\frac{H}{UT}}}\\
    &= O\pa{b\sqrt{UTH} \log \pa{\frac{H}{UT}}},
\end{align*}
where the second inequality follows by Jensen's inequality: the quantity is maximized when all contexts are observed an equal number of times.

\subsection{Proof of \Cref{lem:model2skyline}}

\begin{theorem*}[Skyline Regret]
An agent with full knowledge of the source and target matrices $A$ and $B$ may attain regret
\[R_H \le O\pa{(1 - \kappa)b \sqrt{H} \log\pa{\frac{H}{U}}}. \]
\end{theorem*}

To reduce regret given the decomposition \eqref{eq:normal}, we consider a modification of the standard LinUCB algorithm. Observe that we may rewrite the reward (as given in \eqref{eq:reward}) for a given action $\B x \in \mc A_t$ as
\begin{equation}
    \begin{split}
        X_h &= \cyc{\B x , B_{t_h} \B \theta^{u_h}} + \eta_h\\
        &= \cyc{\B x , D_T A \B \theta^{u_h}} + \cyc{D_G^T \B x, \B \psi^{u_h}} + \eta_h
    \end{split}
\end{equation}
The generator transforms the action to a lower-dimensional space $D_G^T \B x \in \R^{(1-\kappa)b}$. To attain low regret in this ``affine'' variant of the standard stochastic linear bandit problem, we compute confidence sets for $\B \psi^u$ in the style of LinUCB using the transformed actions $D_G^T \B x$ and translated rewards $X_h - \cyc{\B x , D_T A \B \theta^{u_h}}$. We then select the action \begin{equation}
    \label{eq:xh}\B x_h := \argmax_{\B x \in \mc A_h} \{\cyc{\B x, D_T A \B \theta^{u_h}} + \operatorname{UCB}_h\pa{D_G^T \B x}\},\end{equation}
where $\operatorname{UCB}_h\pa{D_G^T \B x} \ge \cyc{D_G^T \B x, \B \psi}$ (with high probability) denotes the upper confidence bound for the action at time $h$. We show that this agent achieves lower regret.

We first prove the regret bound in the case in which we only have one user with reward parameter $B \B \theta = D_T A \B \theta + D_G \B \psi$. Observe that as $D_G^T$ has orthogonal rows, for any action $\B x \in \mc A_h$, \[\norm{D_G^T \B x}_2 \le \norm{\B x}_2 \le 1.\] We then make use of the following lemma.
\begin{lemma}[{\cite[Theorems 19.2 and 20.5]{lattimore2020bandit}}]\label{lem:linucbthms}
Let $r_h := \cyc{B \B \theta, \B x_h^* - \B x_h}$ denote the instantaneous regret in round $h  \in [H]$. If $r_h \le \cyc{\tilde {\B \psi_h} - \B \psi, D_G^T \B x_h}$. where $\operatorname{UCB}_h(D_G^T \B x_h) = \cyc{\tilde{ \B \psi_h}, D_G^T \B x_h}$, then,
\[R_H \le O\pa{(1-\kappa)b\sqrt{H} \log\pa{H}}\]
with failure probability $\delta = \frac1H$.
\end{lemma}
We first show that $r_h$ is indeed bounded above by $\cyc{\tilde{\B \psi_h} - \B \psi, D_G^T \B x_h}$. To see this, we first note that by our choice of $\B x_h$,
\begin{align*}
    \cyc{B \B \theta, \B x_h^*} &= \cyc{D_T A \B \theta, \B x_h^*} + \cyc{D_G^T \B x_h^*, \B \psi}\\
    &\le \cyc{D_T A \B \theta, \B x_h^*} + \operatorname{UCB}_h\pa{D_G^T \B x_h^*}\\
    &\le \cyc{D_T A \B \theta, \B x_h} + \operatorname{UCB}_h\pa{D_G^T \B x_h}\\
    &= \cyc{D_T A \B \theta, \B x_h} + \cyc{\tilde{ \B \psi_h}, D_G^T \B x_h},
\end{align*}
where the third line follows from our definition \eqref{eq:xh}.
We may then compute
\begin{align*}
    r_h &= \cyc{B \B \theta, \B x_h^* - \B x_h}\\
    &= \cyc{B \B \theta, \B x_h^*} - \cyc{B \B \theta, \B x_h}\\
    &\le \cyc{D_T A \B \theta, \B x_h} + \cyc{\tilde{ \B \psi_h}, D_G^T \B x_h} - \pa{\cyc{D_T A \B \theta, \B x_h} + \cyc{\B \psi, D_G^T \B x_h}} \\
    &= \cyc{\tilde{ \B \psi_h} - \B \psi, D_G^T \B x_h}.
\end{align*}

Applying \Cref{lem:linucbthms}, we obtain a single-user regret bound
\[R_H \le O\pa{(1-\kappa)b\sqrt{H} \log\pa{H}}.\]
The desired regret bound then follows by applying Jensen's inequality, as in the proof of~\cref{thm:linucbtrivial}.

\subsection{Proof of \Cref{lem:closetospace}}

\begin{lemma*} With probability $1 - \delta$, we have that for all $u$,
\[ L^u \le  \sqrt{2} \pa{1 + \pa{\sigma_{\mathrm{min}}(\pi_A \mc{U}) - \sqrt{2}Z}^{-1}}Z , \] where $q = \frac{\sqrt{bU_0(b + \log(2/\delta))}}{\sqrt{H_0 / U_0} - C(\sqrt{b} + \sqrt{\log(4 / \delta)})} \approx \frac{bU_0}{\sqrt{H_0} - C\sqrt{U_0b}}$; $Z = \frac{q}{\sigma_{\mathrm{min}}\pa{M} - q}$; $M = \tiny{\bmat{A \B \theta^1 &  \hdots & A \B \theta^{U_0}\\ B \B \theta^1 & \hdots &  B \B \theta^{U_0}}}$; $C$ is some absolute constant; and we assume that $\sigma_{\mathrm{min}}\pa{M} > q$.
\end{lemma*}

This proof has two parts: first showing that there exists a point in $\widehat{S}_u$ which is close to $B\theta_u$, and then showing that this implies the result. 

A crucial step in the proof is considering the \emph{sin-theta matrix} between two orthogonal columns $V$ and $\hat{V}$. Suppose that the singular values of $\sin(V, \hat{V})$ are $\sigma_1 \ge \sigma_2 \ldots \ge \sigma_r \ge 0$. Then we call $\sin(V, \hat{V}) = \text{diag}(\cos^{-1}(\sigma_1), \cos^{-1}(\sigma_2), \ldots \cos^{-1}(\sigma_r))$.

\begin{lemma}
With probability $1 - \delta$, there exists $\gamma_u \in \widehat{S}_u$ for all $u \in [U]$ so that \[\norm{\gamma_u - B\B \theta^u} < \sqrt{2} \pa{1 + \pa{\sigma_{\mathrm{min}}(\pi_A \mc{U}) - Z\sqrt{2}}^{-1}}Z .\]
\end{lemma}

\begin{proof}

Note that we have \[\norm{\mc{U} - \widehat{\mc{U}}O} \le \sqrt{2} \norm{\sin(\mc{U}, \widehat{\mc{U}})}\] for a certain orthogonal matrix $O$ (by the first equation under the third statement of \cite[Lemma 1]{cai2018rate}).
Then note that if $\mc{U}_B\B \theta^{u} \in S_u$, we have \[ \pi_B\widehat{\mc{U}}(O\B \theta^{u} + \pi_A\widehat{\mc{U}}^{+}((\pi_A\mc{U} - \pi_A\widehat{\mc{U}}O)\B \theta^{u}) \in S_u' \] which we set to be $\gamma_u$.

Now observe that
\begin{align*}
    &\norm{\pi_B\widehat{\mc{U}}(O\B \theta^{u} + \pi_A\widehat{\mc{U}}^{+}((\pi_A\mc{U} - \pi_A\widehat{\mc{U}}O)\B \theta^{u}) - \pi_B\mc{U} \B \theta^{u}}\\& \le \norm{(\pi_B\widehat{\mc{U}}O - \pi_B \widehat{\mc{U}})\B \theta^{u}} + \norm{\pi_A\widehat{\mc{U}}^{+}((\pi_A\mc{U} - \pi_A\widehat{\mc{U}}O)\B \theta^{u})}  \\
    &\le\sqrt{2} \norm{\sin(\mc{U}, \widehat{\mc{U}})} \pa{\norm{\B \theta^{u}} + \norm{(\pi_A \widehat{\mc{U}})^{+}}\norm{\B \theta^{u}} }.   
\end{align*}

The next result shows that $M$ is close to $\widehat{M}$, which we show later is enough to show that $\mc{U}$ is close to~$\widehat{\mc{U}}$. 

\begin{lemma}
With probability $1 - \delta$, we have \[\norm{\widehat{M} - M} \le q.\]
\end{lemma}

\begin{proof}

Since the matrix formed by stacking the arms together is fixed across users, we have that \[E := M - \widehat{M} = (D^{\top}D)^{-1}D^{\top}\eta = F\eta,\] where $\eta$ is a $(H_0 / U_0) \times U_0$ matrix of i.i.d. $1$-subgaussian random variables. We seek to bound the value of $\norm{E}$. 

\begin{claim}
With probability $1 - \delta$, we have
\[ \norm{E} \le C\norm{F}\sqrt{U_0(b + \log(1 / \delta))}, \]
where $C$ is an absolute constant.
\end{claim}

\begin{proof}
Let $\mc{N}$ be an $\epsilon$-net of the $1$-ball of size $5^{b}$ with radius $1/2$, which exists by  \cite[Corollary 4.2.13]{vershynin2018high}. Then note that $2\sup_{x \in \mc{N}} \norm{E^{\top}x} \ge \norm{E^{\top}}$ by  \cite[Lemma 4.4.1]{vershynin2018high}.
Letting $\ell=C\sqrt{b + \log(1/\delta)}$, note that
\[ \norm{E^{\top}x} = \norm{\eta^{\top}F^{\top}x} = \norm{\eta^{\top}v} \ge C\ell\sqrt{U_0} \norm{v} = C\ell\sqrt{U_0} \cdot \norm{F^{\top}x}, \] with probability at most $ 2e^{-\ell^2} = 5^{-b} \delta$ -- where $C$ is an absolute constant -- from properties of subgaussian random variables. Thus, by a union bound, with probability at least $1 - \delta$ we have \[ \norm{E} / 2 = \norm{E^{\top}} / 2 \le  \sup_{x \in \mc{N}} \norm{E^{\top}x} \le \sup_{x \in \mc{N}} C\ell\sqrt{U_0} \norm{F^{\top}x} \le C\ell\sqrt{U_0} \norm{F}.\qedhere \]

\begin{claim}
With probability $1 - \delta$, we have \begin{equation*}\label{qe:fop}
    \norm{F} \le \frac{1}{\sqrt{H_0 / bU_0} - C(1 + \sqrt{\log(2/\delta)/b})}. \end{equation*}
\end{claim}
\begin{proof} 
We can see that \[ \norm{F} = \sqrt{\norm{FF^{\top}}} = \lambda_{\text{min}}(D^{\top}D)^{-1/2} = 1 / \sigma_{\text{min}}(D). \]

By \cite[Theorem 4.6.1]{vershynin2018high}, since $D$ is $1/\sqrt{b}$ times a $(H_0 / U_0) \times b$ matrix whose entries are uniform Gaussian, we have \[ \sigma_{\text{min}}(D) \ge \sqrt{1/b} (\sqrt{H_0 / U_0} - C(\sqrt{b} + \sqrt{\log(2/\delta)})) = \sqrt{H_0 / bU_0} - C(1 + \sqrt{\log(2 / \delta) / b})\] with probability $1 - \delta$.
\end{proof}

By the second claim, we have that \[ \norm{F} \le \frac{1}{\sqrt{H_0 / bU_0} - C(1 + \sqrt{\log(4/\delta)/b})} \] with probability at least $1 - \delta / 2$. By the first claim, we have that with probability at least $1 - \delta/2$, \[\norm{E} \le C \norm{F} \sqrt{U_0(b+\log(1/\delta))}.\] Thus, by the union bound, both inequalities are true with probability at least $1 - \delta$, in which case we have we have \[ \norm{E} \le C\norm{F}\sqrt{U_0(b+\log(1/\delta))} \le \frac{\sqrt{U_0(b+\log(1/\delta))}}{\sqrt{H_0 / bU_0} - C(1 + \sqrt{\log(4/\delta)/b})},  \]
as desired.
\end{proof}
\end{proof}

Suppose now that $\norm{\widehat{M} - M} \le q$. We  show that the inequality \[ L^u \le  \sqrt{2} \pa{1 + \pa{\sigma_{\mathrm{min}}(\pi_A \mc{U}) - \sqrt{2}Z}^{-1}}Z \] follows, which suffices for the main claim.

By the remark under subsection 2.3 in \cite{cai2018rate}, we have that \[ \norm{\sin(\mc{U}, \widehat{\mc{U}})} \le \frac{q}{\sigma_{r}(\widehat{M})}. \]

Note that $\sigma_{r}(M) \le \sigma_{r}(\widehat{M}) + \norm{E}$ by Weyl's inequality on singular values, so \[\sigma_{r}(\widehat{M}) \le \sigma_r(M) - \norm{E} \le \sigma_r(M) - q,\] so we have \[ \norm{\sin(\mc{U}, \widehat{\mc{U}})} \le Z .\] 

Then note that \[ \norm{(\pi_A\widehat{\mc{U}})^{+}} = (\sigma_{\text{min}}(\pi_A\widehat{\mc{U}}))^{-1}. \]

Recall that by the first equation under the third statement of \cite[Lemma 1]{cai2018rate}  there exists an orthogonal square matrix $O$ so that \[\norm{\mc{U} - \widehat{\mc{U}}O} \le \sqrt{2} \norm{\sin(\mc{U}, \widehat{\mc{U}})}.\] Thus, 
\begin{align*}
\sigma_{\text{min}}(\pi_A\widehat{\mc{U}}) &= \sigma_{\text{min}}(\pi_A\widehat{\mc{U}}O) \\ 
&\ge \sigma_{\text{min}}(\pi_A\mc{U}) - \norm{\pi_A\mc{U} - \pi_A \widehat{\mc{U}}O} \\
&\ge \sigma_{\text{min}}(\pi_A \mc{U}) - \norm{\mc{U} - \widehat{\mc{U}}O} \\
&\ge \sigma_{\text{min}}(\pi_A \mc{U}) - \sqrt{2} \norm{\sin(\mc{U}, \widehat{\mc{U}})} \\
&\ge \sigma_{\text{min}}(\pi_A \mc{U}) - \sqrt{2} Z,
\end{align*}
where the second line follows from Weyl's Inequality.

Finally, we have \[ \norm{\widehat{U}_A^{+}} \le  \pa{\sigma_{\text{min}}(U_A) - Z\sqrt{2}}^{-1} \]
and we may now conclude. 
\end{proof}
From the lemma, we have shown that with probability $1 - \delta$, there exists $\gamma_u \in \widehat{S}_u$ for all $u \in [U]$ so that \[\norm{\gamma_u - B\B \theta^u} \le \sqrt{2} \pa{1 + \pa{\sigma_{\mathrm{min}}(\pi_A \mc{U}) - Z\sqrt{2}}^{-1}}Z .\]
It remains to show that $L^{u} \le \norm{\gamma_u - B\B \theta^{u}}$ for all $u$. Note that $\gamma_u$ is an element of $E_u$. Thus it suffices to show that $L^{u}$ is the distance between $B\theta$ and the closest element in $E_u$, which is given by $W_{\parallel}W_{\parallel}^{\top}B\theta$. Thus the distance is given by \[ \norm{B\B \theta^{u} - W_{\parallel}W_{\parallel}^{\top}B \B \theta^{u}}_2 = \norm{W_{\perp}W_{\perp}^{\top}B\B \theta^{u}}_2 = \norm{W_{\perp}^{\top}B\B \theta^{u}}_2 = \norm{(W^{\top}B\B \theta^{u})_{(1-\kappa)b + 2:b}}_2. \]
and we may conclude. 

\subsection{Proof of \Cref{thm:giga}}

\begin{theorem*} With an oblivious exploration policy that draws identical arms for each user from $\frac{1}{\sqrt{b}}\mc{N}(0,I_b)$,   assuming that each user appears equally many times, with probability at least $1 - \pa{U\delta' + \delta}$, we have
\begin{align*}
    \sum_{u \in [U]} \mathrm{LowOFUL}^u_{H_u} &\le O\biggr( (1-\kappa)b \sqrt{2\log (1 / \delta')}\sqrt{UH}\log\pa{\frac{H}{U}}\\
    & \quad \quad \quad + H\sqrt{2} \pa{1 + \pa{\sigma_{\mathrm{min}}(\pi_A \mc{U}) - \sqrt{2}Z}^{-1}}Z \log\pa{\frac{H}{U}} \biggr ).
\end{align*}
\end{theorem*}

Here, we first give the LowOFUL algorithm in Algorithm~\ref{alg:LowOFUL}. We note here that we use a generalization of the LowOFUL algorithm where the action sets are allowed to vary with time obliviously; we note that the argument in the proof does not change. 

Note that we can break up the total regret into two terms, which are
\[ 
    R_H^{(1)} = \ex{\sum_{u \in [U]} \sum_{h \le  H_0 : u_h = u} \pa{\max_{\B x \in \mc A_h} \cyc{\B x, B\B \theta^u} - X_h}} = \mathrm{Explore}_{H_0}
\]
and
\begin{align*}
    R_H^{(2)} = \ex{\sum_{u \in [U]} \sum_{h > H_0 : u_h = u} \pa{\max_{\B x \in \mc A_h} \cyc{\B x, B\B \theta^u} - X_h}} &= \ex{\sum_{u \in [U]} \sum_{h > H_0 : u_h = u} \pa{\max_{\B x \in W^{\top} \mc A_h} \cyc{\B x, W^{\top} B\B \theta^u} - X_h}} \\ 
    &= \sum_{u \in [U]} \ex{\sum_{h > H_0 : u_h = u} \pa{\max_{\B x \in W^{\top} \mc A_h} \cyc{\B x, W^{\top} B\B \theta^u} - X_h}} .
\end{align*}

We give our bound in terms of $L_u$, and then substitute our bound from \Cref{lem:closetospace} to conclude. Note that each user gets $(H - H_0)/U$ timesteps in the second phase, which is less than $H / U$.  

From \cite[Appendix D(c)]{jun2019bilinear} (and using their notation) we have that (for $\lambda = 1$) \[ R_{H/U} \le 2 \sqrt{\beta_{H/U}} \sqrt{\log \frac{|V_{H/U}|}{|\Lambda|}} \sqrt{H/U}.  \] 

Now, we can bound $\log \frac{|V_{H/U}|}{|\Lambda|} \le 2((1 - \kappa)b + 1) \log(1 + H/U)$, and  \[ \sqrt{\log \frac{|V_{H/U}|}{|\Lambda|}} \le \sqrt{2((1 - \kappa)b + 1) \log(1 + H/U)}. \]

By our choice of $W$, combined with \Cref{lem:closetospace}, we now choose $\beta_{H/U}$ so that 
\begin{align*} 
\sqrt{\beta_{H/U}} &= \sqrt{ \log \frac{|V_T|}{|\Lambda|} + 2\log\pa{\frac{1}{\delta'}}}  + 1 + \sqrt{\lambda_{\perp}} L_u\\
&= O\pa{ \sqrt{2(1 - \kappa)b \log(1 + H/U)} + \sqrt{2\log \frac{1}{\delta'}} + L_u\sqrt{\frac{H/U}{ (1-\kappa)b }\log(1 + H/U)}}, 
\end{align*}
where we have taken $\lambda_{\perp} = \frac{H}{ (1-\kappa)b + 1}\log(1 + H/U)$. 

Finally, by the main result of \cite{jun2019bilinear} we obtain that the regret, with probability at least $1 - \delta'$, for a particular user, is 

{\small 
\begin{align*}
    &O\pa{ \sqrt{2(1 - \kappa)b \log(1 + H/U)} + \sqrt{2\log \frac{1}{\delta'}} + L_u\sqrt{\frac{H/U}{ (1-\kappa)b }\log(1 + H/U)}}\sqrt{2((1 - \kappa)b + 1) \log(1 + H)}\sqrt{H/U} \\
    &= O(2(1-\kappa)b \log(1 + H/U) \sqrt{H/U}\sqrt{2\log(1 / \delta')} + (H/U)L_u\log(1 + H/U)).
\end{align*}}

By a union bound, with probability at least $1 - U\delta' - \delta$, we have that $L_u$ is small and that all of these regrets are small, which results in an overall regret given by the above expression. Multiplying by $U$ and dividing by $H$ now yields the desired result. 

\subsection{Proof of \Cref{lem:model3skyline}}
\begin{lemma*}With access to all task matrices, we obtain a regret of \[ R_H \le O\pa{r\sqrt{UH}\log\pa{\frac{H}{U}}} .\]\end{lemma*}

Once given access to the task matrices $B_1, \dots, B_T$, the problem reduces to a contextual bandit problem with just $U$ contexts, instead of $U \cdot T$ contexts. Using the task matrices, we first compute the $\mc U_t$
\[\bmat{\mc{U}_1 & \pa{\mc{U}_1}_{\bot}\\
\vdots & \vdots\\
\mc{U}_{T} & \pa{\mc{U}_{T}}_{\bot}}
\bmat{\Sigma_r & 0 \\ 0 & 0} \bmat{\mc{V}_1 & \pa{\mc{V}_1}_{\bot}\\
\vdots & \vdots\\
\mc{V}_{T} & \pa{\mc{V}_{T}}_{\bot}} =: \bmat{B_1 \\ \vdots \\ B_T}\]
By the definition of $r$, we may write $B_t \B \theta^u = \mc U_t \B \phi^u$ for all $(u,t) \in [U] \times [T]$, where $\B \phi^u \in \R^r$. Noting that 
\[\cyc{\B x, B_t\B \theta^u} = \cyc{\B x, \mc U_t\B \phi^u} = \cyc{\mc U_t^\top \B x, \B \phi^u},\]
we consider a LinUCB agent that runs one bandit per user, transforming the action space by $\mc A_h \mapsto \mc U_{t_h}^\top\mc A_h$, where $t_h$ denotes the task context at time $h$. From here, we obtain regret in a manner similar to the argument presented in the proof of \cref{thm:linucbtrivial} where the number of contexts $U\cdot T$ is replaced with $U$, and the dimensionality $b$ is replaced with $r$. This yields the desired bound 
\[ R_H \le O\pa{r\sqrt{UH}\log\pa{\frac{H}{U}}} .\]

\pdfoutput=1

\section{Pseudocode for Algorithms}

\subsection{Algorithm for One Source, One Target, Many Users}

\begin{algorithm}[H]\label{alg:model2}
\SetKwInput{KwInput}{Input}
\SetKwInput{KwOutput}{Output}
\SetKwInput{Pull}{Pull arm}
\SetAlgoLined
\KwInput{An exploration strategy $\mathtt{Explore}: (u,h) \mapsto \B x \in \mc A_h$ and the LowOFUL algorithm $\mathtt{LowOFUL}: (\mc A_h, L, h) \mapsto \B a \in \mc A_h$ where $L$ bounds the almost sparse components}
\KwOutput{Arms to pull $\B x_h \in \mc A_h$ for $h \in [H]$}
\tcc{Use the exploration strategy for the first $H_0$ iterations on the $U_0$ beta group users.}
\For{$h \in [H_0]$}{
Observe context $u_h \in [U_0]$\;
\Pull{$\B x_h := \mathtt{Explore}(u_h, h)$}
}

\tcc{Learn the transformer and generator from observed data}
\For{$u \in [U_0]$}{
Estimate $\widehat{B \B \theta^u}$ from the arm-reward history\;
}
$\bmat{\widehat{\mc{U}} & \widetilde{\mc{U}}} \bmat{\Sigma_r & 0 \\ 0 & \Sigma_{\kappa b}} \bmat{\widehat{\mc{V}^T} \\ \widetilde{\mc{V}^T}} := \bmat{A \B \theta^1 &  \hdots & A \B \theta^{U_0}\\ \widehat {B \B \theta^1} & \hdots & \widehat{ B \B \theta^{U_0}}}$ \tcp*{Compute the SVD}
$\widehat{D_T} := \pi_B \widehat{\mc U}\pa{\pi_B \widehat{\mc U}}^+$\tcp*{Approximate the transformer}
$\widehat{D_G} := \mathtt{OrthogonalBasis}\pa{\pi_B \widehat{\mc U}\nul\pa{\pi_A \widehat{\mc U}}}$\tcp*{Approximate the generator}
\tcc{Compute the almost low-dimensional bases}
\For{$u \in [U_0]$}{
$W^u := \bmat{W^u_{\parallel} & W^u_{\perp}} := \mathtt{OrthogonalBasis}\pa{\operatorname{span}\pa{\widehat{D_T} A \B \theta^u,  \widehat{D_G}}}$\;
Compute the upper bound on $L^u$ as in Lemma \ref{lem:closetospace}\;
}
\tcc{Apply $\mathtt{LowOFUL}$ with one bandit per context using the learned bases}
\For{$h \in [H_0 + 1, H]$}{
Observe context $u_h \in [U]$ and decision set $\mc A_h$\;
$\B a_h := \mathtt{LowOFUL}\pa{\pa{W^{u_h}}^T \mc A_h, L^u, h}$\tcp*{Run $\mathtt{LowOFUL}$ in the learned basis}
\Pull{$\B x_h := W^{u_h}\B a_h$}
}
\caption{One Target, Many Users}
\end{algorithm}

\subsection{Algorithm for No Source, Many Targets, Many Users}

\begin{algorithm}[H]\label{alg:lowOFUL}

\SetKwInput{KwInput}{Input}
\SetKwInput{KwOutput}{Output}
\SetKwInput{Pull}{Pull arm}
\SetAlgoLined
\KwInput{An exploration strategy (e.g. LinUCB) $\mathtt{Explore}: (u,h) \mapsto \B x \in \mc A_h$ and the LinUCB algorithm $\mathtt{LinUCB}: (\mc A_h, h) \mapsto \B a \in \mc A_h$}
\KwOutput{Arms to pull $\B x_h \in \mc A_h$ for $h \in [H]$}
\tcc{Phase 1: Use the exploration strategy for the first $H_0$ iterations on the $U_0$ and $T_0$ beta group users and tasks.}
\For{$h \in [H_0]$}{
Observe context $u_h \in [U_0]$\;
\Pull{$\B x_h := \mathtt{Explore}(u_h, h)$}
}
\tcc{Learn the matrices $\set{\widehat{\mc U_t}}_{t \in [T_0]}$}
\For{$(u,t) \in [U_0] \times [T_0]$}{
Estimate $\widehat{B_t \B \theta^u}$ from the arm-reward history\;
}
$\bmat{\widehat{\mc{U}_1} & \widetilde{\mc{U}_1}\\
\vdots & \vdots\\
\widehat{\mc{U}_{T_0}} & \widetilde{\mc{U}_{T_0}}}
\bmat{\Sigma_r & 0 \\ 0 & \Sigma_{bT_0 - r}} \bmat{\widehat{\mc{V}_1} & \widetilde{\mc{V}_1}\\
\vdots & \vdots\\
\widehat{\mc{V}_{T_0}} & \widetilde{\mc{V}_{T_0}}} := \bmat{\widehat {B_1 \B \theta^1} & \hdots & \widehat{ B_1 \B \theta^{U_0}}\\
\vdots & \vdots\\
\widehat {B_{T_0} \B \theta^1} & \hdots & \widehat{ B_{T_0} \B \theta^{U_0}}
}$ \tcp*{Compute SVD}
\BlankLine
\BlankLine
\tcc{Phase 2: All users arrive, but only request recommendations for the $T_0$ beta tasks}
\For{$h \in [H_0+1, H_1]$}{
Observe context $u_h \in [U_0]$\;
\Pull{$\B x_h := \mathtt{Explore}(u_h, h)$}
}
\tcc{Use arm-reward history and the $\set{\widehat{\mc U_t}}_{t \in [T_0]}$ to estimate $\set{\widehat{\B \phi^u}}_{u \in [U]}$}
\For{$u \in [U]$}{
Estimate $\widehat{\B \phi^u}$ with least-squares using the matrices $\set{\widehat{\mc U_t}}_{t \in [T_0]}$, as described in \Cref{sec:model3}.
}
\BlankLine
\BlankLine
\tcc{Phase 3: Use our estimates and LinUCB to perform well when all users arrive and request recommendations for any task}

\For{$h \in [H_1 + 1, H]$}{
Observe context $u_h \in [U]$ and decision set $\mc A_h$\;
\tcc{Use the estimate $\widehat{\B \phi^u}$ to run $\mathtt{LinUCB}$ where the reward parameter is given by $\operatorname{vec}(\mc U_t)$}
$\operatorname{vec}\pa{\B A_h} := \mathtt{LinUCB}\pa{\mc A_h \pa{\widehat{\B \phi^u}}^\top, h - H_1}$\;
\Pull{$\B x_h := \frac{\B A_h\widehat{\B \phi^u}}{\norm{\widehat{\B \phi^u}}^2}$}
}
\caption{No Sources, Many Targets, Many Users}
\label{alg:model3}
\end{algorithm}

\subsection{The LowOFUL Algorithm}
For completeness, we overview the LowOFUL Algorithm \cite{jun2019bilinear} here.

\begin{algorithm}[H]\label{alg:LowOFUL}
\SetKwInput{KwInput}{Input}
\SetKwInput{KwOutput}{Output}
\SetKwInput{Pull}{Pull arm}
\SetAlgoLined
\KwInput{$\mc A_h$, $L$, $h$, and a regularization constant $\lambda$}
\KwOutput{Arm to pull $\B a_h \in \mc A_h$}
$\lambda_{\bot} := \frac{h}{r\log\pa{1 + \frac{h}{\lambda}}}$\;
$\Lambda := \operatorname{diag}\pa{\underbrace{\lambda, \dots, \lambda}_{\text{first $r$ entries}}, \lambda_{\bot}, \dots, \lambda_{\bot}}$ \tcp*{Diagonal matrix for regularization}
$A := $the actions taken up to time $h-1$ -- one per row\;
$\B y :=$the rewards received up time $h-1$ -- one per row\;
$\widetilde{\B \theta} := \argmin_{\B \theta} \frac12 \norm{\B A \B \theta}^2 + \frac{1}{2}\norm{\B \theta}_{\Lambda}^2$\tcp*{Estimate of the reward parameter}
$V_{h-1} := \Lambda + A^TA$\;
$\beta_{h-1} := \pa{\sqrt{\log{\frac{|V_{h-1}|}{|\Lambda|\delta^2}}} + \sqrt{\lambda} + L\sqrt{\lambda_{\bot}}}^2$\tcp*{Length of the ellipsoid axes}
$c_{h-1} := \set{\B \theta \mid \norm{\B \theta - \widetilde{\B \theta}}_{V-1} \le \sqrt{\beta_{h-1}}}$\tcp*{The $1-\delta$ confidence set}
$\B a_h := \argmax_{\B a \in \mc A_h} \max_{\B \theta \in c_{h-1}} \cyc{\B \theta, \B a}$\;
\Pull{$\B a_h$}
\caption{LowOFUL}
\end{algorithm}

\newpage

\section{Full Experimental Details}\label{app:runtime}

Actions are isotropic Gaussian random vectors of length $a$ and $b$ respectively.
\begin{lemma}\label{lem:runtime}
The run-time of the first algorithm is in total \[ \underbrace{H_0 (Kb^2 + b^3) }_{\text{Phase 1}} + \underbrace{(H-H_0)(K((1-\kappa)b)^2 + ((1-\kappa)b)^3)}_{\text{Phase 2}} + \underbrace{(a+b)^2U_0 + U_0^3}_{\text{SVD}} \]  and the run-time of the second algorithm is \[ \underbrace{H_0 (Kb^2 + b^3) }_{\text{Phase 1}} + \underbrace{(H_1 - H_0)(Kn^2 + n^3)}_{\text{Phase 2}} + \underbrace{(H - H_1)(K(bn)^2 + (bn)^3)}_{\text{Phase 3}} +  \underbrace{(T_0b)^2U_0 + U_0^3}_{\text{SVD}}. \] 
 \end{lemma}
 
 \paragraph{Experiments with One Source, One Target, Many Users}

In this experiment, we choose $\kappa = 0.9$, $a = 20$, $b = 20$, $U_0 = 25$, $U = 500$, with a choice between $40$ actions. We also have $H_0 = 2,000$ and $H = 8,000$. We generate $A$ and $B$ by letting $\B \phi^{u}$ be unit isotropic dimension-$d$ Gaussians, where $d = a + b(1-\kappa)$, and letting $A$, $B$, be $1/\sqrt{d}$ times arbitrary $a \times d$ and $b \times d$-dimensional elementwise Gaussian matrices. In addition, we let $\mc{A}$ be comprised of $|\mc{A}|$ dimension-$a$ isotropic Gaussian vectors, and use Gaussian noise for $\eta$. 

The only remaining issue is the order in which users arrive. 
First, we have the users come in order $1, 2, \ldots U_0$ for $H_0 / U_0$ entries, and then the users arrive in random order. 

In LowOFUL, we take $\lambda = 1$ and $\lambda_{\perp} = \frac{H / U}{ ((1-\kappa)b + 1)\log(1 + H/U)}$, again following \cite{jun2019bilinear}.

\paragraph{Experiments with No Sources, Many Targets and Many Users}

In this experiment, we choose $b = 3$, $T = 30$, $U = 50$, $r = 6$, $|\mc{A}| = 40$ actions, and $T_0 = 3$, tasks and $U_0 = 5$ users in the reduced phase, where the matrices $B_i$ and vectors $\B \phi^u$ are all given by isotropic standard Gaussian matrices and vectors. Moreover, the $\B \theta^u \in \R^{d}$ are each standard isotropic Gaussian, and each task matrix is $1/\sqrt{b}$ times an arbitrary $b \times r$ elementwise uniform matrix. We run Algorithm~\ref{alg:model3} for $H_0 = 1,000$ actions in the first phase, $H_1 - H_0 = 3,000$ steps in the second phase, and $H - H_1 = 9,000$ actions in the third phase.

\pdfoutput=1

\section{Difficulties in the Analysis of Algorithm 2}\label{app:hard}

The intuition and motivation behind algorithm 2 (the proposed algorithm in section 4) mirrors that of algorithm 1 (our solution to the setup in section 3). However, the analysis becomes intractable due to a dependence structure that arises in this setting that is not present in the setting of section 3. Whereas we had access to the true optimal arms on the source tasks for members of the beta group in section 3, we only have access to confidence sets for the true optimal arms for the other target tasks in this setting. These confidence sets are stochastically determined by the algorithm in a way that depends on the action history. This, in and of itself, is not a problem so long as one can make $\ell_2$ distance guarantees on the diameter of the confidence set. However, the axes of these confidence sets are not of equal length (i.e. the confidence sets are anisotropic), and the lengths themselves are dependent on the action history. Thus, although the algorithm is well-motivated given the analysis of section 3, and the fundamental ideas remain the same, we believe the analysis of the second case to be intractable. We were, however, able to demonstrate the success of the method experimentally.

\end{document}